\documentclass[10pt]{article}

\usepackage[a4paper,textwidth=18cm,textheight=24cm,top=2.85cm, bottom=2.85cm, left=1.5cm, right=1.5cm]{geometry}

\usepackage{icdp2009}

\makeatletter
\long\def\@makecaption#1#2{%
\vskip\abovecaptionskip
\sbox\@tempboxa{#1. #2}%
\ifdim \wd\@tempboxa >\hsize
#1. #2\par
\else
\global \@minipagefalse
\hb@xt@\hsize{\box\@tempboxa\hfil}%
\fi
\vskip\belowcaptionskip}
\makeatother

\usepackage{lmodern}

\usepackage{graphicx}
\usepackage{times}
\usepackage{hyperref}       % hyperlinks
\usepackage[usenames,dvipsnames,table]{xcolor}
\definecolor{darkblue}{rgb}{0.0,0.0,0.55}
\hypersetup{
  colorlinks = true,
  citecolor  = darkblue,
  linkcolor  = darkblue,
  citecolor  = darkblue,
  filecolor  = darkblue,
  urlcolor   = darkblue,
}

\usepackage{url}            % simple URL typesetting
\usepackage{booktabs}       % professional-quality tables
\usepackage{amsfonts}       % blackboard math symbols
\usepackage{nicefrac}       % compact symbols for 1/2, etc.
\usepackage{microtype}      % microtypography
\usepackage{makecell} 	% linebreak in table
\usepackage{bm}		% bold math fonts

\usepackage{algorithm}% http://ctan.org/pkg/algorithm
\usepackage{cellspace}
\usepackage{algpseudocode}
\usepackage{color}
\usepackage{graphicx}
\usepackage{amsmath,amsthm}

\newcommand{\arccosh}{{\rm arcosh}}
\newcommand{\norm}[1]{\|{#1}\|}

\newcommand{\ip}[2]{\langle {#1},\, {#2} \rangle}

\newtheorem{theorem}{Theorem}[section]
\newtheorem{lem}[theorem]{Lemma}

\begin{document}
\noindent

\bibliographystyle{plain}

\title{Neighborhood Growth Determines Geometric Priors \\ for Relational Representation Learning}

\authorname{Melanie Weber}
\authoraddr{Princeton University \\ mw25@math.princeton.edu}

\maketitle


\abstract
The problem of identifying geometric structure in heterogeneous, high-dimensional data is a cornerstone of representation learning. While there exists a large body of literature on the embeddability of canonical graphs, such as lattices or trees, the heterogeneity of the relational data typically encountered in practice limits the applicability of these classical methods. In this paper, we propose a combinatorial approach to evaluating embeddability, i.e., to decide whether a data set is best represented in Euclidean, Hyperbolic or Spherical space. Our method analyzes nearest-neighbor structures and local neighborhood growth rates to identify the geometric priors of suitable embedding spaces. For canonical graphs, the algorithm's prediction provably matches classical results. As for large, heterogeneous graphs, we introduce an efficiently computable statistic that approximates the algorithm's decision rule. We validate our method over a range of benchmark data sets and compare with recently published optimization-based embeddability methods.

\section{Introduction}
A key challenge in data science is the identification of geometric structure in high-dimensional data. Such structural understanding is of great value for designing efficient algorithms for optimization and learning tasks. Classically, the structure of data has been studied under an Euclidean assumption.
%, for instance, by analyzing feature vectors in $\mathbb{R}^n$. 
The simplicity of vector spaces and the wide range of well-studied tools and algorithms that assume such structure make this a natural approach. However, lately, it has been recognized that Euclidean spaces do not necessary allow for the most `natural' representation, at least not in the low-dimensional regime. Recently, the representation of data in hyperbolic space has gained significant interest~\cite{NK17,chamberlain,sala,poincare-glove}. The intrinsic hierarchical structure of data sets ranging from social networks to wordnets has been related to ``tree-likeness'' and in turn to hyperbolic embeddability, since trees embed with low distortion into hyperbolic space~\cite{sarkar}. On the other hand, there is a long tradition for spherical embeddings in computer vision and shape analysis, where volumetric data is efficiently represented in spherical space.

In this work, we study the question of embeddability in the context of relational representation learning: For a given set of pairwise similarity relations, we want to determine the geometric priors of an embedding space that reflects the intrinsic structure of the data. Optimization-based embeddability methods~\cite{sala, gu} rely on \emph{Multi-dimensional Scaling} (\emph{MDS}), that require performing large-scale minimization tasks to determine suitable embedding parameters. Furthermore, such methods require an a priori fixed embedding dimension. Here, we introduce a purely combinatorial approach, that efficiently determines suitable priors through a direct analysis of the data's discrete geometry without a priori dimensionality assumptions. 

The paper is structured as follows: We will first analyze the relation between embeddability and neighborhood growth rates. Expanding neighborhoods (\emph{exponential} growth) exhibit tree-like properties and contribute to the hyperbolicity of the data set. On the other hand, cycles have a contracting effect, as they slow down local growth. Therefore, slowly expanding neighborhoods are an indicator of good embeddability into Euclidean (\emph{linear} growth) or Spherical space (\emph{sublinear} growth). To extend this framework from canonical graphs to heterogeneous relational data, we introduce a regularization that ensures uniform node degrees and therefore allows for a direct comparison of growth in diverse graph neighborhoods. We then introduce a statistic (\emph{3-regular score}) that aggregates local growth information across the graph. Based on the 3-regular score, we determine the geometric priors of the most suitable embedding space. For canonical graphs ($N$-cycles, ($\sqrt{N} \times \sqrt{N}$)-lattices and $b$-ary trees) we give a proof that the approach matches classical embeddability results. Furthermore, we establish a relation between the 3-regular score and discrete Ricci curvature~\cite{forman,Ol}, a concept from Discrete Geometry that has been linked to embeddability~\cite{WSJ2}.

The introduced method is purely combinatorial with a computational complexity linear in the average neighborhood size multiplied by the number of nodes. Moreover, as a local analysis, it can be efficiently parallelized. 

\subsection{Related Work}
\textbf{Embeddings for Representation Learning.} 
The theoretical foundation of Euclidean embeddability has been layed out by~\cite{bourgain} and~\cite{JL}. Here, the relation of \emph{intrinsic} and \emph{metric dimension} is of special interest: One can show with a volume argument that data $X$ with intrinsic dimension $dim (X)$ can be embedded with metric dimension $\Theta \left(	 \frac{dim(X)}{\log(\alpha)}\right)$ and distortion $\alpha$ into Euclidean space.~\cite{abraham2} study this relation algorithmically and derive distortion bounds. Recently, optimization-based hyperbolic embeddings have gained a surge of interest in the representation learning community.~\cite{NK17,NK18,chamberlain,poincare-glove} proposed optimization-based frameworks for embedding similarity data into hyperbolic space.~\cite{sala} analyze representation trade-offs in hyperbolic embeddings with varying geometric priors.~\cite{gu} introduce mixed-curvature embeddings by studying embeddability onto product manifolds. A related approach by~\cite{WN18} explores connections between graph motifs and hyperbolic vs. spherical embeddability.

\textbf{Spectral Approaches.}
~\cite{wilson} propose a spectral approach that determines embedding parameters by minimizing the magnitude of the first (spherical case) or second (hyperbolic case) eigenvalue. In this approach, the sign of the curvature is a hyperparameter of the objective and has to be prior known. However, the framework is only valid for isometrically embeddable data, leading to inaccuracies in heterogeneous data. In addition, spectral methods have limited scalability on large-scale data.

\textbf{Discrete Curvature.} 
In addition to spectral approaches, discrete curvature has recently gained interest as a structural graph characteristic. Gromov's $\delta$-hyperbolicity~\cite{gromov}, a discrete notion of sectional curvature, has been used to study the hyperbolicity of relational data~\cite{mahoney,gu}. Discrete notions of Ricci curvature were studied as graph characteristics~\cite{Ol,WSJ1,WSJ2}, providing insight into the local geometry of the underlying relational data.

%\paragraph{MotifCount} 
%[\textbf{omit in initial submission?/ add to camera-ready version}]
%A preliminary heuristic method based on graph motifs was presented in [Anonymous et al.].~\footnote{Anonymized non-archival workshop paper attached.} The method builds on the power of regularization to limit possible graph motifs to four (\emph{leafs}, \emph{chains}, \emph{3-forks}, \emph{stars}) that are characterized by easy-to-analyze expansion or contraction properties. The size and frequency of these motifs is then used as a heuristic for the sign of the curvature, mainly serving as a means to distinguish between hyperbolic and spherical embeddability.

\subsection{Contributions}
In the present paper, we connect the structure of nearest neighbor relations to the intrinsic geometry of relational data. We argue that the growth rates of graph neighborhoods serve as a proxy for the geometric priors of a suitable embedding space. To account for the heterogeneity of complex data sets, we perform a regularization that allows for a low-cost embedding of any graph into a 3-regular graph quasi-isometrically. In the regularized setting, where node degrees are uniform, we directly compare local neighborhood growth rates and deduce the curvature of the embedding space. 

We show that our classification scheme matches theoretical results for canonical graphs. Furthermore, we establish a relation to discrete Ricci curvature. For analyzing complex, heterogeneous data sets as typically encountered in ML applications, we introduce a statistic (\emph{3-regular score}) that aggregates local growth information and approximates the algorithm's decision rule. A series of validation experiments on classic benchmark graphs and real-world data sets demonstrates the applicability of the proposed approach. Finally, we compare our method to recently published embeddability benchmarks, validating that the 3-regular score predicts the lowest-distortion embedding.

%%% Background and Notation %%%
\section{Background and Notation}
\subsection{Model Spaces}
We consider canonical Riemannian manifolds with constant curvature as embedding spaces, which can be characterized through the following set of \emph{model spaces} $\lbrace M_{\kappa}^d \rbrace$:
\begin{enumerate}
\item $M_{0}^n=\mathbb{R}^n$ denotes the canonical \textbf{Euclidean space} with the inner product $\ip{u}{v}_{E}=\sum_{i=1}^n u_i v_i$ that gives rise to the Euclidean norm $\norm{v}_E=\sqrt{v_1^2 + \dots + v_n^2}$ and the metric $d_{E}(u,v)=\norm{u-v}_E$.
\item The \textbf{$n$-sphere} $M_{\kappa>0}^n =  \mathbb{S}^n= \lbrace	 v \in \mathbb{R}^{n+1}: \norm{v}_E = \sqrt{\kappa}	\rbrace$ is an embedded submanifold of the $\mathbb{R}^{n+1}$ with constant positive curvature. A canonical metric is given by $d_S (u,v)={\rm arcos} (\ip{u}{v}_E)$.
\item The \textbf{hyperboloid} $M_{\kappa<0}^n =  \mathbb{H}^n = \lbrace	 v \in \mathbb{R}^{n+1}: \norm{v}_H = \sqrt{-\kappa}, v_1>0	\rbrace$ is 
a manifold with constant negative curvature. It is defined with respect to the \emph{Minkowski inner product}
\begin{align*}
\ip{u}{v}_H &= u^T {\rm diag}(-1, 1, \dots, 1) v \\
&= -u_1 v_1 + u_2 v_2 + \dots + u_n v_n \; ,
\end{align*}
which gives rise to the hyperbolic metric $d_H(u,v)={\rm acosh}(-\ip{u}{v}_H)$ and norm $\norm{u}_H = \sqrt{\ip{u}{v}_H}$.
\end{enumerate}
Here, we focus on the canonical model spaces $M_{0,\pm 1}$. Table~\ref{tab:model-spaces} summarizes important geometric properties that will be used in the following sections. Note that $M_{\pm 1}$ can be easily generalized to arbitrary curvatures ($\vert \kappa \vert \neq 1$) by multiplying the respective distance functions by $\frac{1}{\sqrt{\vert \kappa \vert}}$. In the following, we will drop the subscript $E$ when referring to the Euclidean notions. For a more comprehensive overview on model spaces, see, e.g.~\cite{BH99}.
\begin{table*}[t]
\centering
\begin{small}
\begin{tabular}{llll}
\toprule
 & \textbf{Euclidean} \(\mathbb{R}^d\) & \textbf{Spherical} \(\mathbb{S}^d\) & \textbf{Hyperboloid} \(\mathbb{H}^d\)\\
\midrule
Space & \(\mathbb{R}^n\) & \(\lbrace x \in \mathbb{R}^{n+1} : \ip{x}{x} = 1\rbrace\) & \(\lbrace x \in \mathbb{R}^{n+1} : \ip{x}{x}_H = -1, x_0 > 0 \rbrace\)\\
\(\ip{u}{v}\) & \(\sum_{i=1}^n u_i v_i\) & \(\sum_{i=1}^n u_i v_i\) & \(-u_0 v_0 + \sum_{i=1}^{n} u_i v_i\)\\
\(d(u,v)\) & \(\sqrt{\ip{u - v}{u - v}}\) & \(\arccos(\ip{u}{v})\) & \(\arccosh(-\ip{u}{v}_H)\)\\
Curvature & \(\kappa = 0\) & \(\kappa = 1\) & \(\kappa = -1\)\\
%Sum of angles & \(\pi\) & \(>\pi\) & \(<\pi\)\\
%Circle length & \(C(r) = 2\pi r\) & \(C(r) = 2 \pi \sin r\) & \(C(r) = 2 \pi \sinh r\)\\
%Disc area & \(A(r) = 2 \pi r^2 / 2\) & \(A(r) = 2 \pi(1 - \cos r)\) & \(A(r) = 2 \pi(\cosh r - 1)\)\\
Canonical graph & ($\sqrt{N} \times \sqrt{N}$)-lattice & $N$-cycle & regular $N$-tree\\
\bottomrule
\end{tabular}
\caption{Geometric properties of model spaces.}
\label{tab:model-spaces}
\end{small}
\end{table*}

\subsection{Graph Motifs and Local Topology}
The present paper focuses on relational data, i.e., we assume access to a measure of similarity between any two elements. Natural representations of such data are graphs $G=\lbrace V, E \rbrace$, where $V$ denotes the set of vertices or nodes (representing data points) and $E$ the set of edges (representing relations). Additional features may be given through weights on the edges which we encode in the weight functions $\omega_E(e): E \rightarrow \mathbb{R}$. 
%$\omega_V(v): V \rightarrow \mathbb{R}$ and $\omega_E(e): E \rightarrow \mathbb{R}$. 

The importance of graph motifs for understanding the higher-order structure of graphs has long been recognized and intensely studied~\cite{mahoney,olhede,verbeek-suri}. Motifs are commonly defined as characteristic local connectivity patterns that occur in varying sizes and frequencies. While there is no canonical classification, trees and cycles have emerged as prevalent motifs in the study of network topology, due to having the greatest topological stability (i.e., the highest Euler characteristic)~\cite{olhede}. A random walk initiated at the root of a tree will never return to its point of origin, but expand into space. On the other hand, a random walk within a cycle is guaranteed (or, in a circle with outgoing connections, likely) to return to its origin, introducing a local contraction. This naturally relates to local growth rates in graph neighborhoods: While trees intrinsically encode exponential growth, cycles introduce a local contraction, resulting in sublinear growth rates. We will connect these ideas with the problem of embeddability.

\subsection{Embeddability}
An \emph{embedding} between metric spaces $(X_1,d_1)$ and $(X_2,d_2)$ is described as a map $\phi: X_1 \rightarrow X_2$. Here, we consider embeddings of relational data into canonical model spaces, i.e. we want to embed the graph metric $(V,d_G)$ of a (weighted or unweighted) graph $G=\lbrace V, E \rbrace$ using a map $\phi: V \rightarrow \mathcal{M}_{0, \pm 1}^d$, where $d_G$ denotes the usual path distance metric. The goodness of an embedding is measured in terms of distortion. We denote the \emph{additive} distortion $c_A \geq 0$ of the map $\phi$ as
\begin{align*}
\vert d_G (u,v) - d_{\mathcal{M}} (\phi(u), \phi(v)) \vert \leq c_A \quad \forall  u,v \in V \; ,
\end{align*}
and the \emph{multiplicative distortion} $c_M \geq 0$ as
\begin{align*}
d_{\mathcal{M}} (\phi(u), \phi(v)) &\leq d_G (u,v) \\ 
&\leq c_M d_{\mathcal{M}} (\phi(u), \phi(v)) \quad \forall  u,v \in V \; .
\end{align*}
Note that for an isometric map $c_A = 0$ and $c_M =1$. 

While little is known about the embeddability of large, heterogeneous graphs, there exists a large body of literature on the embeddability of canonical graphs. The, to our knowledge, best known results for multiplicative distortion are summarized in Table~\ref{tab:embed}. In the following we develop a computational method that applies not only to canonical graphs, but to any relational data set.
\begin{table*}[t]
\centering
\begin{small}
\begin{tabular}{llll}
\toprule
 & \textbf{Euclidean} \(\mathbb{R}^d\) & \textbf{Spherical} \(\mathbb{S}^d\) & \textbf{Hyperboloid} \(\mathbb{H}^d\)\\
\midrule
($\sqrt{N} \times \sqrt{N}$) - lattice & $c_M \leq O(1) \; ^{a)}$ & - & $c_M \geq O(\sqrt{N}/ \log(N))  \; ^{a)}$ \\
$N$-cycle & $c_M \leq O(1) \; ^{a)}$ & $c_M \leq O(1)  \; ^{a)}$ & $c_M \geq O(N/ \log(N))  \; ^{a)}$ \\
$b$-regular tree (size $N$) & $c_M \leq O(N^{\frac{1}{d-1}}) \;^{c)}$ & - & $c_M \leq O(1+\epsilon)\; ^{b)}$ \\
\bottomrule
\end{tabular}
\caption{{\small Known results on embeddability of canonical graphs. For $^{a)}$ see~\cite{verbeek-suri}, for $^{b)}$ see~\cite{sarkar} and $^{c)}$~\cite{gupta99}.}}
\label{tab:embed}
\end{small}
\end{table*}
%

%%% Methods %%%
\section{Methods}
We determine the geometric priors of a suitable embedding space with a two-step method, (1) by performing a \emph{regularization} that enforces uniform node degrees while preserving structural information and (2) by analyzing local neighborhood growth rates to determine the dominating geometry (\emph{3-regular score}).

\subsection{Regularization}
\begin{figure*}[t]
\centering
\includegraphics[width=0.6\textwidth]{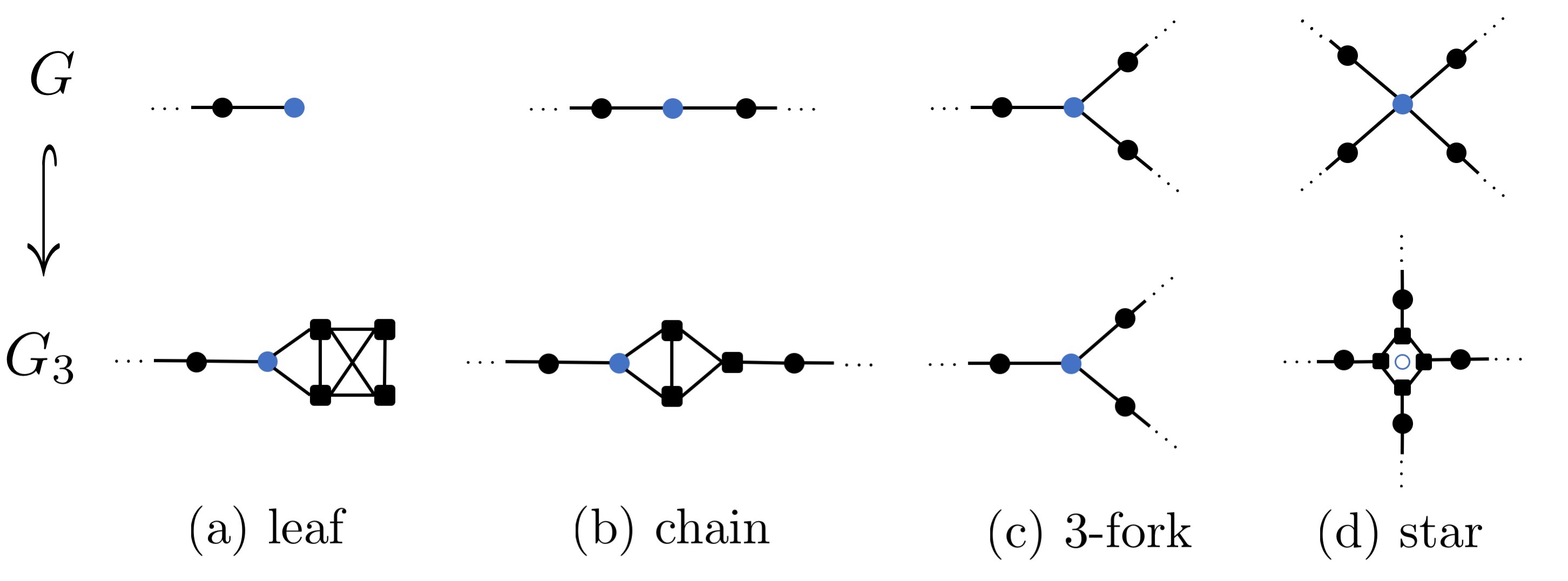}
\caption{Regularization $G \hookrightarrow G_3$ following~\cite{bermudo}. A vertex $v \in V$ is called \emph{leaf}, if it has degree 1, \emph{chain} if it has degree 2, \emph{3-fork}, if it has degree 3 and \emph{star} otherwise. The transformation maps each vertex and the edges connecting to its direct neighbors onto the respective three-regular structure as predicted by the vertex' degree. Circles represent original vertices, squares auxiliary vertices that have been added to enforce uniform node degrees.}
\label{fig:G-G3}
\end{figure*}
Relational data as typically encountered in data science applications is very heterogeneous, making it difficult to draw a conclusion on the global geometry from local analysis. Our first step is therefore the regularization of the graph's connectivity structure that will allow for a more efficient comparison of local neighborhood growth rates and, in turn, the local geometry. We will use throughout the paper the following (conventional) notation: When analyzing the neighborhood of a vertex $v \in V$, we say that $v$ is the \emph{root} or \emph{center} of the neighborhood. Neighborhood directionality is always assumed from the root outwards, the start of an outward facing edge is called \emph{parent}, the end \emph{child}. The root has no parent.

We utilize a quasi-isometric embedding~\cite{bermudo} that allows for embedding any (connected) graph into a three-regular graph, i.e. a graph with uniform node degrees (${\rm deg}(v)=3$ for all $v \in V$). The regularization algorithm is shown schematically in Fig.~\ref{fig:G-G3}, for more details see Appendix A. One can show the following bound on the distortion induced by this transformation:
\begin{theorem}[~\cite{bermudo}]\label{thm:bermudo}
$G \hookrightarrow^\phi G_3$ is a $(\epsilon + 1, \epsilon)$-quasi-isometric embedding, i.e. $c_M=O(1+ \epsilon)$ and $c_A=O(\epsilon)$. 
\end{theorem}

\subsection{Estimating local neighborhood growth rates}
In order to decide the geometry of a suitable embedding space, we want to analyze \emph{neighborhood growth rates}. Consider first a continuous, metric space ($\mathcal{X},d_{\mathcal{X}}$). The $\delta$-neighborhood of a point $x \in \mathcal{X}$ is defined as the set of points within a distance $\delta$, i.e. 
\begin{equation}
\mathcal{B}_\delta(x) = \lbrace y \in \mathcal{X}: \; d_{\mathcal{X}}(x,y) \leq \delta \rbrace \; .
\end{equation} 
In Euclidean space, the volume of $\mathcal{B}_\delta(x)$ is growing at a polynomial rate in $\delta$. However, in hyperbolic space, the volume growth is exponential. Therefore, the local volume growth of neighborhoods serves as a proxy for the space' global geometry.

In discrete space, instead of analyzing volume growth, we characterize the local growth of neighborhoods. We denote the $r$-neighborhood of a vertex $v \in V(G)$ as
\begin{equation}
\mathcal{N}_r (v) = \lbrace u \in V(G): d_G (u,v) \leq r	\rbrace \; .
\end{equation} 
We say that $\mathcal{N}_r (v)$ is \emph{exponentially expanding}, if it grows exponentially in $r$ and \emph{linearly expanding}, if it grows at least linearly in $r$. Otherwise, we call $\mathcal{N}_r (v)$ \emph{sublinearly expanding}. Thanks to the regularized structure of our graphs, we can quantify precisely the corresponding neighborhood growth laws:
\begin{enumerate}\label{list:growth-rates}
\item \emph{exponentially expanding}:\\ $\gamma_{EE}(v,R) = \vert v \vert  + \sum_{r=1}^R 3 \cdot 2^{r-1}$;
\item \emph{linearly expanding}:\\ $\gamma_{LE} (v,R) = \vert v \vert + \sum_{r=1}^R 3 r$. %$\gamma_E (R) = 1 + \sum_{r=1}^R 3 \cdot 2 \cdot \frac{r}{2}$.
\end{enumerate}
Here, $\vert v \vert$ denotes the size of the root structure, i.e., $\vert v \vert = 1$, if ${\rm deg} (v) \leq 3$ and $\vert v \vert = {\rm deg}(v)$ otherwise, due to the transformation of star-nodes into three-regular rings (see Fig.~\ref{fig:G-G3} and Appendix A). Then we say that the $R$-neighborhood of a vertex $u \in V(G)$ is exponentially expanding, if $\vert \mathcal{N}_R (u) \vert \geq \gamma_{EE}(R)$, linearly expanding, if $\vert \mathcal{N}_R (u) \vert \geq \gamma_{LE}(R)$ and sublinearly expanding otherwise. 
%We summarize this in the following definition:
%
%\begin{defn}[Neighborhood growth rates]
%The $R$-neighborhood of a vertex $v \in V$ is growing
%\begin{enumerate}
%\item exponentially, if $\vert \mathcal{N}_R (u) \vert \geq \gamma_{LE}(R)$
%\end{enumerate}
%\end{defn}
For canonical graphs, we get the following neighborhood growth rates (a proof can be found in Appendix B):
\begin{theorem}[Neighborhood growth in canonical graphs]\label{prop:growth-rates}
For $1 \leq R \ll N$, every $R$-neighborhood in (i) a $b$-regular tree is exponentially expanding,
(ii) an ($\sqrt{N} \times \sqrt{N}$)-lattice is linearly expanding and
(iii) an $N$-cycle is sublinearly expanding.
%\begin{enumerate}
%\item a $b$-regular tree is exponentially expanding;
%\item an ($\sqrt{N} \times \sqrt{N}$)-lattice is linearly expanding;
%\item an $N$-cycle is sublinearly expanding.
%\end{enumerate}
\end{theorem}
Utilizing the link between neighborhood growth and global geometry that we discussed above, we introduce the following decision rule for the geometric prior of the embedding space (${\rm sign}(\kappa)$):
\begin{itemize}\label{eq:expand-contract}
\item If $\mathcal{N}_R(v)$ is \emph{exponentially expanding} $\forall v \in V$ and $1 \leq R \ll N$, assume ${\rm sign}(\kappa)<0$, i.e., embed into $\mathbb{H}^d$;
\item If $\mathcal{N}_R(v)$ is \emph{lineraly expanding} $\forall v \in V$ and $1 \leq R \ll N$, assume ${\rm sign}(\kappa)=0$, i.e., embed into $\mathbb{R}^d$;
\item If $\mathcal{N}_R(v)$ is \emph{sublinearly expanding} $\forall v \in V$ and $1 \leq R \ll N$, assume ${\rm sign}(\kappa)>0$, i.e., embed into $\mathbb{S}^d$.
\end{itemize}
Note that these result match known embeddability results for canonical graphs (see Table~\ref{tab:embed}). In the following section we introduce a statistic that allows for applying this decision rule to heterogeneous relational data also.

\subsection{3-Regular Score}
The heterogeneity commonly encountered in relational data makes it impossible to generalize global growth rates from a local analysis as in Thm.~\ref{prop:growth-rates}. We will typically find a mixture of local growth rates, that is not covered by the decision rule above. Instead, we focus on determining the globally \emph{dominating} geometry: We analyze growth rates locally and then compute an average across the graph, weighted by the size of the respective $R$-neighborhood. The resulting statistic, to which we refer as the \emph{3-regular score}, can be computed as follows:
\begin{align*}
A &=\sum_{v \in V} \sigma(v) \vert \mathcal{N}_R(v) \vert \; ; \\
\sigma(v) &=\begin{cases}
1, &\mathcal{N}_R(v) \; {\rm sublinearly \; expanding} \\
-1, &\mathcal{N}_R(v) \; {\rm exponentially \; expanding} \\
0, &{\rm otherwise} \; .
\end{cases}
\end{align*}
To determine the geometric priors, we apply the following decision rule:
\begin{itemize}
\item if $A>0$, assume ${\rm sign}(\kappa)>0$, i.e., embed into $\mathbb{S}^d$;
\item if $A < 0$, assume ${\rm sign}(\kappa)<0$, i.e., embed into $\mathbb{H}^d$;
\item and if $A \approx 0$, assume ${\rm sign}(\kappa)=0$, i.e., embed into $\mathbb{R}^d$.
\end{itemize}
For weighted networks, we perform the same regularization and computation of the 3-regular score, but replace in the growth rate estimations $3$ with the weighted node degree of the center. When determining $\vert \mathcal{N}_R(v) \vert$, i.e., the set of all neighbors up to distance $R$ from $v$, the metric is the weighted path distance $\tilde{d}_G$. Consequently, we count all neighbors $v'$ with $\tilde{d}_G(v,v') \leq R$.

The decision rule is motivated by locally aggregating neighborhood growth information. 
Hereby $\sigma$ encodes whether the neighborhood growth is locally \emph{exponential} (indicating hyperbolic space, i.e., 
${\rm sign}(\kappa) = -1$ and $\sigma=-1$), \emph{linear} (indicating Euclidean space, i.e., $\kappa = 0$ and $\sigma=0$) or \emph{sublinear}
(indicating spherical space, i.e., ${\rm sign}(\kappa)=1$ and $\sigma=-1$). 
Due to the heterogeneity of the graphs, we weigh the $\sigma$s by the size of the neighborhood ($\vert \mathcal{N}_R(v) \vert$) to give large neighborhoods a larger influence on the overall score. This is motivated by the fact that the "amount" of distortion incurred is proportional to the largest subgraph of another space's canonical motif: For instance, when embedding a graph into hyperbolic space, distortion is proportional to the size of the largest cycle by a Steiner node construction (see, e.g.,~\cite{verbeek-suri}). The resulting 3-regular score $A$ after reweighing will then depend on the size of the graph, in particular on the number of edges in the regularized graph $G_3$. Therefore, we normalize by dividing by the the number of edges in $G_3$, i.e., we compare $A/\# E(G_3)$ across data sets. The dependency of $A$ on $R$ is explicitly given through the weights $\vert \mathcal{N}_R(v) \vert$; $R$ is upper-bounded by the diameter of the graph. 

%% Connection to curvature
\subsection{Comparison with other discrete curvatures}
The 3-regular score is conceptually related to discrete notions of curvature, such as Gromov's $\delta$-hyperbolicity~\cite{gromov} or discrete Ricci curvature~\cite{forman,Ol}. While Gromov's $\delta$ captures by construction the hyperbolicity of a graph, discrete Ricci curvature is not restricted to negative values. In this section, we analyze the relationship between the 3-regular score and discrete Ricci curvature and compare the suitability of both concepts to measure embeddability.

In the following, we will think of data sets as graphs, where nodes represent the data points and edges the pairwise similarities between them. For simplicity, we only consider unweighted graphs. We consider both \emph{Ollivier-Ricci curvature} (${\rm Ric_O}$)~\cite{Ol} and \emph{Forman-Ricci curvature} (${\rm Ric_F}$)~\cite{forman,WSJ1} that have previously been analyzed in the context of large-scale data and complex networks. Although both curvatures are classically defined edge-based, we will use node-based expressions. Those can be derived by defining the Ricci curvature at a node as the aggregate curvature of its incoming and outgoing edges (see Appendix C for more details). For ${\rm Ric_O}$, consider
\begin{align*}
{\rm Ric_O} (u,v) = 1 - W_1 (m_u, m_v) \; ,
\end{align*}
where $W_1 (m_u, m_v)$ is the Wasserstein-1 distance that measures the cost of transporting mass from $u$ to $v$. $m_u=\frac{1}{{\rm deg}(u)}$ denotes the uniform measure on the neighborhood of $v$. The corresponding node-based notion is given by
\begin{align*}
{\rm Ric_O} (v) = \frac{1}{{\rm deg}} \sum_{(u,v)} {\rm Ric_O}(u,v) \; .
\end{align*}
${\rm Ric_F}$ of an edge $(u,v)$ is defined as
\begin{align*}
{\rm Ric_F} (u,v) = 4 - {\rm deg}(u) - {\rm deg}(v) \; ,
\end{align*}
the corresponding node-based expression as
\begin{align*}
{\rm Ric_F}(v) = \frac{1}{{\rm deg}} \sum_{(u,v)} {\rm Ric_F}(u,v) \; .
\end{align*}
%4 - {\rm deg}(v) - \sum_{(u,v)} \frac{{\rm deg}(u)}{{\rm deg}(v)} 
The aggregation of local growth rates in the 3-regular score resembles the Ricci curvature's property of "locally averaging" sectional curvature. Note that both Ricci curvatures encode only structural information of the first and second neighbors, whereas the 3-regular score measures structural information from neighbors up to distance $R$.  Both ${\rm Ric_F}$ and the 3-regular score are very scalable due to their simple combinatorial notion. ${\rm Ric_O}$ has limited scalability on large-scale data, since the computation of Wasserstein distances requires solving a linear program for every edge. 

To evaluate whether Ricci curvature can select a suitable embedding space, we consider again canonical graphs. Due to their regular structure, the global average curvature is equal to the local curvature at any node in the graph. We derive the following results:\\

\begin{theorem}
\label{thm:ric}
At any node $v$, we have
\begin{enumerate}
\item ${\rm Ric_O}(v)<0$ and ${\rm Ric_F}(v)<0$ in a $b$-regular tree,
\item ${\rm Ric_O} \leq 0$ and ${\rm Ric_F}(v)<0$ in an ($\sqrt{N} \times \sqrt{N}$)-lattice, and
\item ${\rm Ric_O}=0$ and ${\rm Ric_F}=0$ in an $N$-cycle.
\end{enumerate}
\end{theorem}
The proof follows from combinatorial arguments and (for ${\rm Ric_O}$) curvature inequalities~\cite{jost-liu}; it can be found in Appendix C. The theorem shows that Ricci curvature, similar to Gromov's $\delta$, correctly detects hyperbolicity, but cannot characterize structures with non-negative curvature. We conclude, that Ricci curvature is not suitable for model selection and that the 3-regular score has broader applicability.

\section{Experiments}
We have shown above, that in the case of canonical graphs, our approach's prediction matches known embeddability results. In this section, we want to experimentally validate that the \emph{3-regular score} determines suitable embedding spaces for complex, heterogeneous data. In the following, we report ``normalized'' 3-regular scores, meaning that we divide by the number of edges in the regularized graph multiplied by the mean edge weight. This adjusts for differences in the average neighborhood size and therefore allows for a comparison across data sets of varying sizes.
\begin{figure*}[t]
\begin{center}
%\begin{minipage}[h]{0.47\textwidth}
% \def\svgwidth{3.8in}
% \includegraphics[scale=0.4]{fig1.pdf} 
%%\resizebox{!}{0.3 \paperhight}{\input{fig1.pdf}} %
%    \label{fig:3reg-synth}
%    \hspace{0.05\textwidth}
%\end{minipage}
%\begin{minipage}[h]{0.47\textwidth}
%  \def\svgwidth{2.3in}
%  \includegraphics[scale=0.4]{benchmark.pdf} 
%     \label{fig:3reg-benchmark}
%\end{minipage}
\includegraphics[scale=0.17]{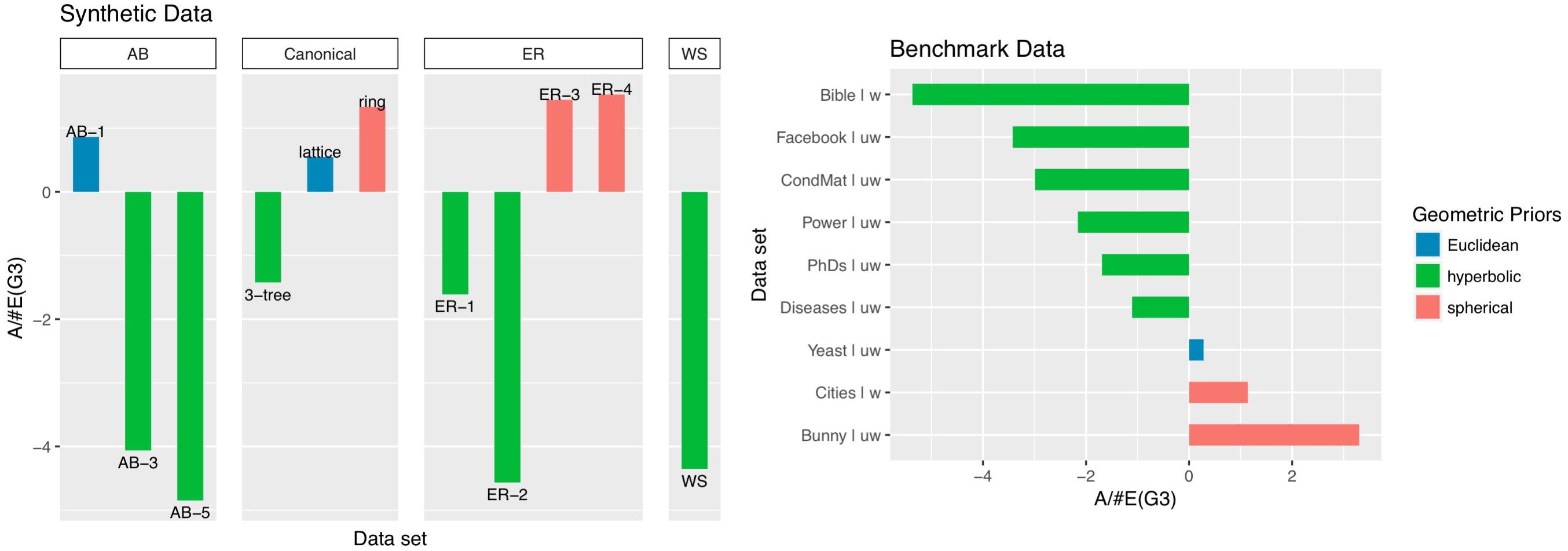} 
\caption{3-regular scores for synthetic graphs and relational benchmark data (w: weighted, uw: unweighted).}
 \label{fig:3reg}
\end{center}
\end{figure*} 
\paragraph{Data sets} We test our method on both synthetic graphs with known embeddability properties and benchmark data sets. For the former, we create data sets of similar size ($N := \vert V \vert \approx 1,000$) to allow for direct comparison.
First we generate an $N$-\textsc{Cycle}, an ($\sqrt{N} \times \sqrt{N}$)-\textsc{lattice} and a ternary \textsc{tree} ($b=3$) with $N$ nodes. We further sample from three classic network models: The random graph model (\textsc{ER})~\cite{ER}, the small world model (\textsc{WS})~\cite{WS} and the preferential-attachment model (\textsc{AB})~\cite{AB} with different choices of hyperparameters. We sample ten networks each and report the average 3-regular score to account for structural sampling variances. Next, we analyze some classic benchmark graphs (both weighted and unweighted) which were downloaded from the Colorado Networks Index~\cite{colorado-index}. The \textsc{Bunny} data was downloaded from the Stanford 3D Scanning Repository~\cite{bunny}. Finally, for validating our approach against recently published embeddability results, we analyze data sets used in~\cite{gu,NK17}, downloaded from the given original sources. We evaluate geographic distances between North American cities (\textsc{Cities}~\cite{cities}), PhD student-adviser relationships (\textsc{PhD},~\cite{phd}) and a citation networks (\textsc{CondMat},~\cite{snap}). %and \textsc{HEP}~\citep{snap}). 
\textsc{Cities} contains similarity data from which we created a nearest-neighbor graph, maintaining edges to the top $5 \%$ neighboring cities.

\paragraph{Results} 3-regular scores for all data sets are shown in Fig.~\ref{fig:3reg}. For canonical graphs, the 3-regular score matches both the theoretical results of our growth rate analysis (Thm.~\ref{prop:growth-rates}) and embeddability results in the literature (Tab.~\ref{tab:embed}). It is well known that \textsc{ER} undergoes phase transitions as the edge threshold increases.~\cite{mahoney} show that \textsc{ER} is not hyperbolic in the low edge threshold regime, but that hyperbolicity emerges with the giant component due to its locally tree-like structure. We analyze \textsc{ER} shortly above (\textsc{ER-3}) and below (\textsc{ER-4}) the giant threshold ($p=\frac{1}{N}$) as well as shortly above (\textsc{ER-2}) and below (\textsc{ER-1}) the connectivity threshold ($p=\frac{\log N}{N}$). Consistent with the theoretical result of~\cite{mahoney}, we observe hyperbolicity, if there is a large giant component (\textsc{ER-1}) or if the graph is connected (\textsc{ER-2}). For \textsc{AB} with linear attachment ($m=1$), the 3-regular score predicts a Euclidean embedding space to be most suitable, as opposed to hyperbolic embeddings for the case of superlinear attachment ($m>1$). This is again consistent with theoretical results on phase transitions in the \textsc{AB} model. The presence of detectable network communities, as found for instance in social networks, has been repeatedly linked to a locally tree-like structure~\cite{adcock,krioukov}. In agreement with this, the 3-regular score predicts good hyperbolic embeddability for both \textsc{WS} (a classic model for studying community structure) and the social network data sets \textsc{Facebook} and \textsc{PhD}. The wordnet \textsc{Bible} was found to embed best into hyperbolic space, in-line with the tree-likeness of such intrinsically hierarchical data.~\cite{sala} observed ``less hyperbolicity'' in biological networks, which matches our results for \textsc{Diseases} and \textsc{Yeast}. Finally, \textsc{Bunny}, a classic benchmark for spherical embeddings, is found to embed best into spherical space.

\paragraph{Validation and comparison with related methods} To validate our results, we compare our predicted geometric priors against recently published embeddability results by~\cite{gu, NK17}. Table~\ref{tab:validation} shows that the 3-regular score predicts the space with the smallest distortion for all benchmark data sets. Here, we follow the authors in reporting \emph{distortion} using the following statistics: The average distortion ${\rm D_{avg}}$, computed over all pairwise distances, and the structural distortion score ${\rm MAP}$ that measures the preservation of nearest-neighbor structures. Isometric embeddability is characterized by ${\rm D_{avg}}=0$ and ${\rm MAP}=1$. For more details, see Appendix D. 
{\renewcommand{\arraystretch}{2.5}
\begin{table*}[t]
\begin{footnotesize}
%\centering
\label{tab:validation}
\begin{tabular}{llllll}
\toprule
 \textbf{Data set} & \textbf{3-regular score} & \textbf{dist} \(\mathbb{R}^{10}\) & \textbf{dist} \(\mathbb{S}^{10}\) & \textbf{dist} \(\mathbb{H}^{10}\) & \textbf{Method} \\
\midrule
Cities & \cellcolor{red!10}1.138 $\rightarrow \mathbb{S}^d$ & $D_{avg}=0.074$ & $\bm{{\rm D_{avg}}=0.060}$ & ${\rm D_{avg}}=0.093$ & Gu et al. \\
PhD students &  \cellcolor{green!10}-1.691 $\rightarrow \mathbb{H}^d$ &  \makecell{$D_{avg}=0.054$\\${\rm MAP}=0.869$} & \makecell{$D_{avg}=0.057$ \\ ${\rm MAP}=0.833$} & \makecell{$\bm{{\rm D_{avg}}=0.050}$ \\ $\bm{{\rm MAP}=0.931}$} & Gu et al. \\
Power & \cellcolor{green!10}-2.158 $\rightarrow \mathbb{H}^d$ &  \makecell{${\rm D_{avg}}=0.092$\\$\bm{{\rm MAP}=0.886}$} & \makecell{${\rm D_{avg}}=0.050$ \\ ${\rm MAP}=0.795$} & \makecell{$\bm{{\rm D_{avg}}=0.039}$ \\ ${\rm MAP}=0.844$} & Gu et al. \\
Facebook & \cellcolor{green!10}-3.423 $\rightarrow \mathbb{H}^d$ &  \makecell{${\rm D_{avg}}=0.065$\\${\rm MAP}=0.580$} & \makecell{${\rm D_{avg}}=0.066$ \\ ${\rm MAP}=0.556$} & \makecell{$\bm{{\rm D_{avg}}=0.060}$ \\ $\bm{{\rm MAP}=0.782}$} & Gu et al. \\
CondMat & \cellcolor{green!10}-2.991 $\rightarrow \mathbb{H}^d$ & ${\rm MAP}=0.356$ & - & $\bm{{\rm MAP}=0.799}$ & Nickel, Kiela \\
%HEP & \cellcolor{green!10} ... $\rightarrow \mathbb{H}^d$ & ${\rm MAP}=0.434$ & - & $\bm{{\rm MAP}=0.811}$ & Nickel, Kiela \\
\bottomrule
\end{tabular}
\end{footnotesize}
\caption{{\small Comparison with recently published benchmark embeddability results~\cite{gu,NK17}.}}
\end{table*}
}

\paragraph{Hyperparameters} The size of the local neighborhoods $\mathcal{N}_R(v)$ over which we compute the 3-regular score (determined by the neighborhood radius $R$) is the central hyperparameter in our analysis. Choosing $R$ too small might leave us with too little information to properly evaluate growth rates, whereas a large $R$ limits scalability. First, note that for $R<3$, $\gamma_{EE} = \gamma_{LE}$ in the regularized graph $G_3$. However, for $R\geq 3$ we always have $\gamma_{EE} \neq \gamma_{LE}$. Consequentially, we require $R \geq 3$. Next, we investigated experimentally if an analysis with larger neighborhood radii reveals additional geometric information by computing 3-regular scores for three data sets with different predicted geometric priors for $R \in \lbrace 3,4,5,10\rbrace$. For the ($\sqrt{N} \times \sqrt{N}$)-\textsc{lattice}, the 3-regular score predicts uniformly ${\rm sign}(\kappa)=0$, i.e. $\mathbb{R}^d$ to be the most suitable embedding space. For \textsc{WS} we observe ${\rm sign}(\kappa)=-1$ across all choices of $R$, predicting hyperbolic embeddability. Finally, for \textsc{Cities} we observe ${\rm sign}(\kappa)=1$ across the different neighborhood radii. In consequence, the 3-regular scores reported above are all computed for $R=3$ to maximize scalability.

%%% Discussion %%%
\section{Discussion}
In this paper, we introduced a framework for determining a suitable embedding curvature for relational data. Our approach evaluates local neighborhood growth rates based on which we approximate suitable embedding curvatures. We provide theoretical guarantees for canonical graphs and introduce a statistic that efficiently aggregates local growth information, rendering the method applicable to heterogeneous, large-scale graphs (both weighted and unweighted). Moreover, we compare the 3-regular score with commonly used notions of discrete curvature in terms of their ability to measure embeddability. We find that discrete curvature is suitable for detecting hyperbolicity, but not for approximating non-negative sectional curvature. This implies that the 3-regular score is better suited for model space selection.\\
Contrary to related embeddability methods, our approach is purely combinatorial, circumventing the need to solve costly large-scale optimization problems. Furthermore, the method does not make any a priori assumption on the dimensionality of the embedding space as opposed to related approaches that impose dimensionality constraints or fix the dimension of the target space. Additionally, the locality of the approach confines the analysis to a small subset of the graph at any given time, allowing for a simple parallelization of the method. This increases the algorithm's scalability significantly.\\
Our results tie into the more general problem of finding data representations that reflect intrinsic geometric and topological features. This problem is three-fold: It requires us to determine (i) the sign of the curvature, which in turn determines the model space, i.e. whether to embed in hyperbolic ($\mathbb{H}^d$), spherical ($\mathbb{S}^d$) or Euclidean space ($\mathbb{R}^d$). Furthermore, (ii) the value of the curvature, which determines local and global geometric parameters of the embedding space, such as distance and angle relations in geodesic triangles and lastly (iii) the dimension of the embedding space. The present work mostly focuses on (i) as we restrict our analysis to canonical Riemannian manifolds with constant sectional curvature ($\kappa \in \lbrace \pm 1, 0 \rbrace$). By combining this approach with MDS-style embedding methods~\cite{gu,sala,bronstein} we could determine the value of the curvature (problem (ii)) also. Hereby, prior knowledge of the sign of the curvature determines the metric of the target space (with the curvature value as hyperparameter) and therefore a suitable objective function to feed into an MDS-style framework. Such a pre-analysis with combinatorial methods should significantly narrow down the search space of suitable curvature values and therefore reduce the overall computational cost. The investigation of such extensions is left for future work.

\section*{Acknowledgements}
The author thanks Maximilian Nickel for helpful discussions at early stages of this project.

%%%%%%%%%%%%%%%%%%
\bibliographystyle{plainnat}
\bibliography{ref}
\clearpage

\onecolumn{
\appendix
\section{Regularization}
We perform a regularization of our graphs (see section 3) which allows for an efficient characterization of local neighborhood growth rates and in turn, the local geometry. In this appendix, we provide more details on the implementation and the theoretical guarantees of the regularization. 

Fig.~\ref{fig:G-G3} shows the regularization schematically. For each vertex $v$ in the graph, we enforce a uniform node degree of 3 within its 1-hop neighborhood $\mathcal{N}_1(v)$. Hereby auxiliary vertices $a$ are inserted and modified edges reweighed. The regularization algorithm is given in Alg.~\ref{alg.reg}.
\begin{algorithm}[ht]
 \caption{Regularization}
  \label{alg.reg}
  \begin{algorithmic}[1] 
  \State Input: $G=\lbrace V(G), E(G) \rbrace$
  \State $\epsilon \gets \max_{e \in E} \omega(e)$
	\For {$v \in V$}		
			\State $\mathcal{N}(v) \gets \lbrace u \in V: v \sim u \rbrace$ \Comment{Neighborhood of $v$.}
			\If {${\rm deg}(v)==1$} \Comment{Leaf: $\mathcal{N}(v)=\lbrace u \rbrace$}
				\State \textsc{Create}($a^0,a^1,a^2,a^3$) \Comment{auxiliary nodes}
				\State $\omega(u,v) \gets \omega(u,v)/2$
				\State $\omega(v,a^0), \; \omega(v,a^1) \gets \epsilon/4$
				\State $\omega(a^i,a^{i+k}) \gets \epsilon/4$ for $i=0,1$; $k=1,2$
			\ElsIf {${\rm deg}(v)==2$} \Comment{Chain: $\mathcal{N}(v)=\lbrace u_1, u_2 \rbrace$}
				\State \textsc{Create}($a^0,a^1,a^2$) \Comment{auxiliary nodes}
				\State $\omega(u_1,v) \gets \omega(u_1,v)/2$
				\State $\omega(v,a^i) \gets \omega(u_2,v)/2$ for $i=0,1$
				\State $\omega(u_2,a^2) \gets \omega(u_2,v)/2$
				\State $\omega(a^0,a^1), \; \omega(a^0,a^2), \; \omega(a^1,a^2) \gets \epsilon/4$
			\ElsIf {${\rm deg}(v)==3$} \Comment{3-regular: $\mathcal{N}(v)=\lbrace u_1, u_2, u_3 \rbrace$}
				\State \textsc{Continue}
			\Else \Comment{Star: $\mathcal{N}(v)=\lbrace u_i \rbrace_{i=1}^{{\rm deg}(v)}$}
				\State \textsc{Create}($a^1,\dots,a^{{\rm deg}(v)}$) \Comment{auxiliary nodes}
				\For {$i = 1, \dots, {\rm deg}(v)$}
					\State  $\omega(a^i,u^i) \gets \omega(u_i,v)/2$
					\State  $\omega(a^i,a^{i+1}) \gets \epsilon/4$
				\EndFor	
			\EndIf
		\EndFor
   \end{algorithmic}
\end{algorithm}
The regularization allows for a quasi-isometric embedding of any graph into a 3-regular graph. We provide the theoretical reasoning below:
\begin{theorem}[Bermudo et al.~\cite{bermudo}]
$G \hookrightarrow^\phi G_3$ is a $(\epsilon + 1, \epsilon)$-quasi-isometric embedding, i.e.
\begin{equation}
d_{G_3}(\phi(x),\phi(y)) \leq (\epsilon + 1)d_G(x,y) + \epsilon \; .
\end{equation}
\end{theorem}
From this we can derive the following additive distortion:
\begin{equation*}
\vert  d_{G_3} - d_G \vert \leq \vert (\epsilon + 1) d_G + \epsilon - d_G \vert = \vert \epsilon (\underbrace{d_G}_{\leq {\rm diam}(G)} + 1) \vert  \leq O(\epsilon) \; ,
\end{equation*}
For the multiplicative distortion we have
\begin{equation*}
d_G \leq  d_{G_3} \leq (1 + \epsilon) d_G + \epsilon \\
\Rightarrow \frac{1}{1+ \epsilon} d_G \leq d_{G_3} \leq (\underbrace{d_G}_{\leq {\rm diam}(G)} + 1) (1+\epsilon) \leq O(1+\epsilon) \; .
\end{equation*}
This gives $c_A=O(\epsilon)$ and $c_M=O(1+\epsilon)$ as given in the main text.
%

%%% Why regularization important %%%%
%As argued in the main paper, the regularization separates motifs allowing for a more accurate estimation of neighborhood growth rates. Here, we provide more details and examples that would be falsely classified by our algorithm without regularization.
%[\textbf{Example} for wrong classification without regularization.]

\section{Neighborhood growth rates in canonical graphs}
We restate the result from the main text:
\begin{theorem}[Neighborhood growth of canonical graphs]
For $3 \leq R \ll N$, every $R$-neighborhood in 
\begin{enumerate}
\item a $b$-regular tree is exponentially expanding;
\item an ($\sqrt{N} \times \sqrt{N}$)-lattice is linearly expanding;
\item an $N$-cycle is sublinearly expanding.
\end{enumerate}
\end{theorem}
\begin{proof}
\begin{figure}[ht]
\centering
\includegraphics[width=0.6\textwidth]{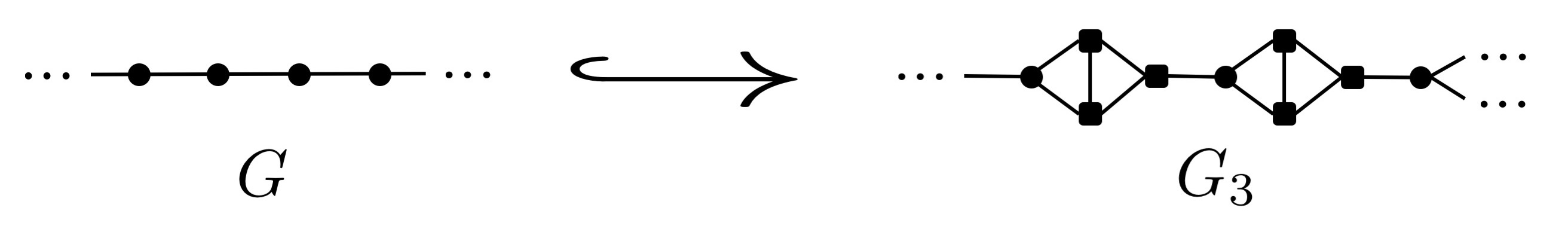}
\caption{Regularization of $N$-cycle.}
\label{fig:reg-cycle}
\end{figure}
The regularization introduces edge weights (see Alg.~\ref{alg.reg}), however, due to the periodic structure of the canonical graphs, those weights are uniform (for lattices and trees) or up to an additive error of $\frac{\epsilon}{4}$ uniform (for cycles). Therefore, we can renormalize the edge weights and analyze the regularized graphs as unweighted graphs. For cycles, the residual additive error does not affect the neighbor count and can therefore be neglected.

Consider first (c) an $N$-cycle. Due to the periodic structure of the \emph{chains} in the regularized graph (see Fig.~\ref{fig:reg-cycle}), there are always either two or three vertices at a distance $r$ from the root, in particular, we have
\begin{align*}
\vert \mathcal{N}_R(v) \vert &= 1 + \sum_{r=1}^R \alpha (r) \; , \\ 
\alpha(r) &=\begin{cases}
2, &{\rm mod(r,3)=0} \\
3, &{\rm else}
\end{cases} \; .
\end{align*}
It is clear that $\sum_{r=1}^R \alpha (r) < 3 R$, i.e. the growth is sublinear.

\begin{figure}[ht]
\centering
\includegraphics[width=0.4\textwidth]{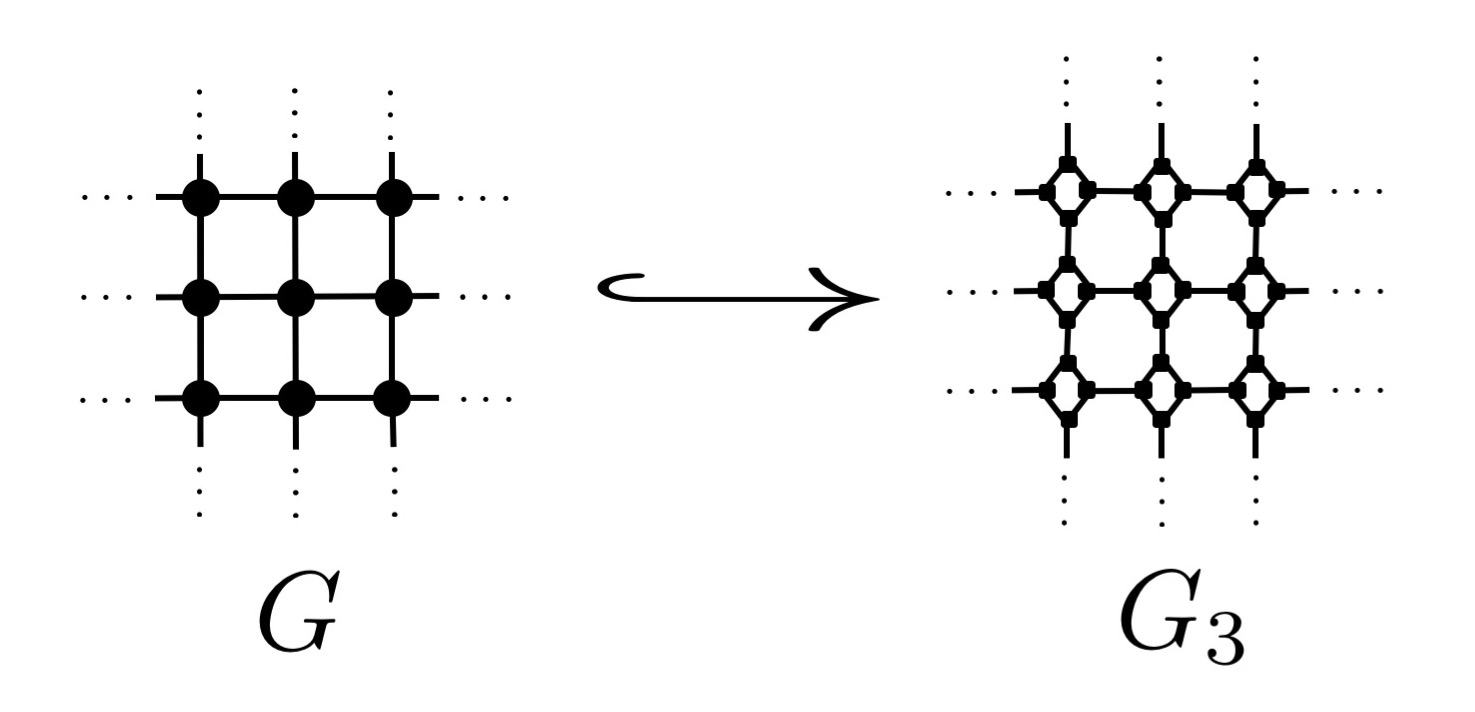}
\caption{Regularization of ($\sqrt{N} \times \sqrt{N}$)-cycle.}
\label{fig:reg-lat}
\end{figure}
Next, consider (b) a ($\sqrt{N} \times \sqrt{N}$)-lattice. From the periodicity of the regularized graph (see Fig.~\ref{fig:reg-lat}), we see that the number of nodes at distance $r$ from the root grows linearly in $r$. In particular, we have
\begin{align*}
\vert \mathcal{N}_R(v) \vert &= \vert v \vert + \sum_{r=1}^R {\rm deg}(v) \alpha(r) \; , \\ 
\alpha(r) &=\begin{cases}
r-1, &{\rm mod(r,3)=1} \\
r, &{\rm else}
\end{cases} \; .
\end{align*}
Since ${\rm deg}(v)=4$, we have ${\rm deg}(v) \alpha(r) \geq 3r$, i.e. the lattice expands linearly.
\begin{figure}[ht]
\centering
\includegraphics[width=0.5\textwidth]{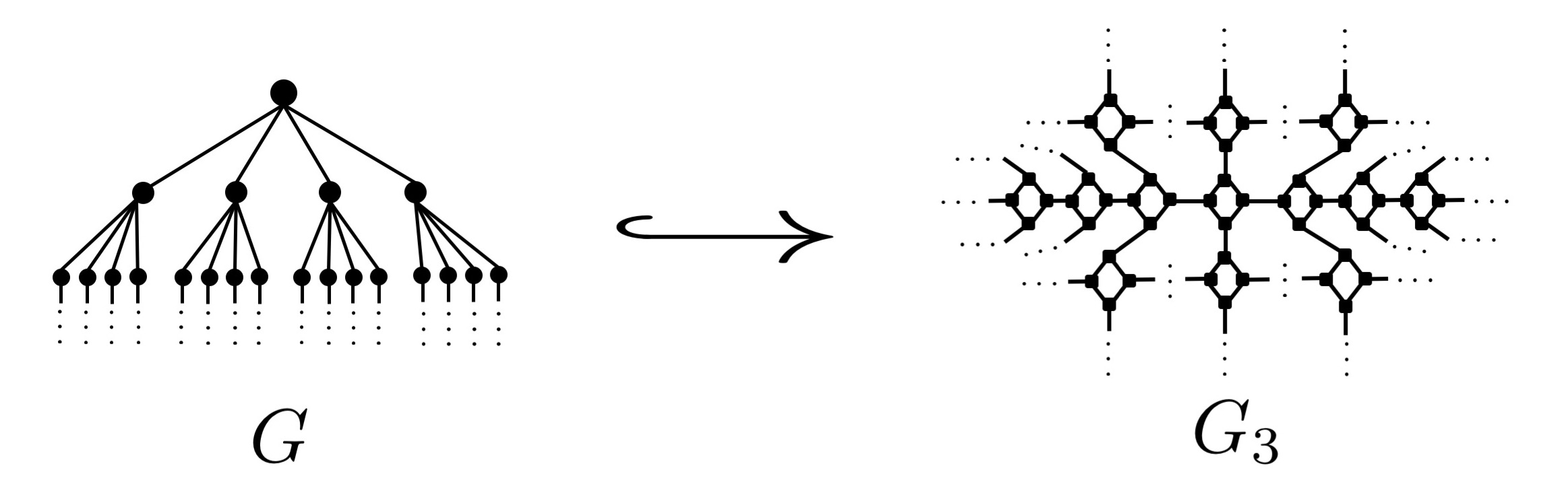}
\caption{Regularization of $b$-ary tree (here $b=4$).}
\label{fig:reg-tree}
\end{figure}
Finally, consider (a) a $b$-ary tree. We first consider the case $b=3$, i.e., a ternary tree. Note that this structure is invariant under our regularization. Since every node has exactly two children, we get the following growth rate:
\begin{align*}
\vert \mathcal{N}_R (v) \vert =  1 + \sum_{r=1}^R 3 \cdot 2^{r-1} = \gamma_{EE} \; ,
\end{align*}
i.e., the ternary tree expands exponentially. Now consider general $b$-ary trees with $b>3$. The building blocks of the periodic structure unfolding are $b$-rings (see Fig.~\ref{fig:reg-tree}). On each level, the nodes have either two children or two nodes have a combined three children, if the $b$-rings close. This results in the following growth rate:
\begin{align*}
\vert \mathcal{N}_R(v) \vert &= \vert v \vert + \sum_{r=1}^R b \cdot \alpha(r) \; , \\
 \alpha(r)&=\begin{cases}
\frac{3}{2} \cdot 2^{r-1}, &{\rm mod(\lceil \frac{r}{2}	\rceil,b)=b-1} \\
2^r, &{\rm else}
\end{cases} \; .
\end{align*}
Since $3 < b$ and $2^{r-1} < \alpha(r)$, all $b$-regular trees expand exponentially.
\end{proof}

\section{Comparison with other discrete curvatures}
We restate the result from the main text (Thm.~\ref{thm:ric}):
\begin{theorem}
At any node $v$, we have
\begin{enumerate}
\item ${\rm Ric_O}(v)<0$ and ${\rm Ric_F}(v)<0$ in a $b$-regular tree,
\item ${\rm Ric_O} \leq 0$ and ${\rm Ric_F}(v)<0$ in an ($\sqrt{N} \times \sqrt{N}$)-lattice, and
\item ${\rm Ric_O}=0$ and ${\rm Ric_F}=0$ in an $N$-cycle.
\end{enumerate}
\end{theorem}
Before proving the theorem, recall the node-based curvature notions for $v \in V(G)$:
\begin{align*}
{\rm Ric_O}(v) &= \frac{1}{{\rm deg}(v)} \sum_{(u,v)} {\rm Ric_O}(u,v) = \frac{1}{{\rm deg}(v)} \sum_{(u,v)} 1 - W_1(m_u, m_v) \\
{\rm Ric_F}(v) &= \frac{1}{{\rm deg}(v)} \sum_{(u,v)} {\rm Ric_F}(u,v) = 4 - {\rm deg}(v) - \sum_{(u,v)} \frac{{\rm deg(u)}}{{\rm deg}(v)} \; .
\end{align*}
Furthermore recall the following curvature inequalities for ${\rm Ric_O}$:
\begin{lem}\cite{jost-liu}\label{lem:ric-inequ}
${\rm Ric_O}$ fulfills the following inequalities:
\begin{enumerate}
\item If $(u,v)$ is an edge in a tree, then ${\rm Ric_O}(u,v) \leq 0$.
\item For any edge $u,v$ in a graph, we have
\begin{align*}
-2 \left( 1 - \frac{1}{{\rm deg}(u)} - \frac{1}{{\rm deg}(v)} \right)_{+} \leq {\rm Ric_O}(u,v) \leq \frac{\# (u,v)}{\max \lbrace {\rm deg}(u), {\rm deg}(v) \rbrace} \; ,
\end{align*}
where $\#(u,v)$ denotes the number of common neighbors (or joint triangles) of $u$ and $v$.
\end{enumerate}
\end{lem}
\begin{proof}(Thm.~\ref{thm:ric})
Consider first (3) an $N$-cycle. By Lem.~\ref{lem:ric-inequ}(2) we have for any $(u,v)$ on the right hand side ${\rm Ric_O}(u,v) \leq 0$, since a cycle has no triangles. Furthermore,  the left hand side gives ${\rm Ric_O}(u,v) \geq 0$, since ${\rm deg}(v)={\rm deg}(u)=2$. This implies ${\rm Ric_O}(v)=0$. We also have
\begin{align*}
{\rm Ric_F}(v) = 4 - {\rm deg}(v) - \left(	\frac{{\rm deg}(u_1)}{2} + \frac{{\rm deg}(u_2)}{2} 	\right) = 0\; ,
\end{align*}
since ${\rm deg}(v)={\rm deg}(u_i)=2$.

Next, consider (1) a $b$-ary tree. By Lem.~\ref{lem:ric-inequ}(1), we have $Ric_O(u,v) \leq 0$ and therefore $Ric_O(v) \leq 0$. Moreover, we have
\begin{align*}
{\rm Ric}_F(v) = 4 - (b+1) - \frac{(b+1)(b+1)}{b+1} \leq 0 \; ,
\end{align*}
since by construction $b \geq 2$.

Finally, consider (2) an ($\sqrt{N} \times \sqrt{N}$)-lattice. Since the lattice has no triangles, Lem.~\ref{lem:ric-inequ}(2) gives ${\rm Ric_O}(u,v) \leq 0$ for any edge $(u,v)$ and therefore ${\rm Ric_O}(v) \leq 0$. In addition,
\begin{align*}
{\rm Ric}_F(v) = 4 - {\rm deg}(v) - \left(	\frac{{\rm deg}(u_1)}{4} + \frac{{\rm deg}(u_2)}{4} + \frac{{\rm deg}(u_3)}{4} + \frac{{\rm deg}(u_4)}{4} 	\right) \leq 0 \; ,
\end{align*}
since ${\rm deg}(v) = {\rm deg}(u_i) = 4$ for all $i$.
\end{proof}

\section{Embeddability measures}
We evaluate the quality of embeddings using two computational distortion measures, following the workflow in ~\cite{gu,NK17}. First, we report the average distortion
\begin{align}
D_{avg} = \sum_{1 \leq i \leq j \leq n} \left\vert \left(	\frac{d_{\mathcal{M}}(x_i,x_j)}{d_G(x_i, x_j)}	\right)^2 - 1 \right\vert  \; .
\end{align}
Secondly, we report ${\rm MAP}$ scores that measure the preservation of nearest-neighbor structures:
\begin{align}
{\rm MAP} = \frac{1}{\vert V \vert} \sum_{u \in V} \frac{1}{{\rm deg}(u)} \sum_{i=1}^{\vert \mathcal{N}_1(u) \vert} \frac{\vert \mathcal{N}_1(u) \cap R_{u,i} \vert}{\vert R_{u,i} \vert} \; .
\end{align}
Here, $R_{u,i}$ denotes the smallest set of nearest neighbors required to retrieve the $i^{th}$ neighbor of $u$ in the embedding space $M$. One can show that for isometric embeddings, ${\rm D_{avg}}=0$ and ${\rm MAP}=1$.
}
\end{document}